\documentclass[twoside]{article}
\usepackage[accepted]{aiarxiv}
\usepackage[round]{natbib}

\usepackage{hyperref}

\newcommand{\dR}{\mathbb{R}}

\usepackage{subcaption}
\usepackage{centernot}

\newcommand{\independent}{\perp\mkern-10mu\perp}

\input{macrosss.sty}

\begin{document}

\twocolumn[
\aistatstitle{On sample complexity of conditional independence testing with Von Mises estimator with application to causal discovery}

\aistatsauthor{%
  \hspace{-.4cm}Fateme Jamshidi \\
  \hspace{-.4cm}EPFL, Switzerland \\
  \hspace{-.4cm}\texttt{fateme.jamshidi@epfl.ch}\\
  \And
  \hspace{-.3cm}Luca Ganassali \\
  \hspace{-.3cm}EPFL, Switzerland \\
  \hspace{-.3cm}\texttt{luca.ganassali@epfl.ch} \\
  \And
  \hspace{-.4cm}Negar Kiyavash \\
  \hspace{-.4cm}EPFL, Switzerland \\
  \hspace{-.4cm}\texttt{negar.kiyavash@epfl.ch} \\
}

\aistatsaddress{} ]

\begin{abstract}
    Motivated by conditional independence testing, an essential step in constraint-based causal discovery algorithms, we study the nonparametric Von Mises estimator for the entropy of multivariate distributions built on a kernel density estimator. 
    We establish an exponential concentration inequality for this estimator. We design a test for conditional independence (CI) based on our estimator, called VM-CI, which achieves optimal parametric rates under smoothness assumptions. Leveraging the exponential concentration, we prove a tight upper bound for the overall error of VM-CI.
    This, in turn, allows us to characterize the sample complexity of any constraint-based causal discovery algorithm that uses VM-CI for CI tests. To the best of our knowledge, this is the first sample complexity guarantee for causal discovery for continuous variables.
    Furthermore, we empirically show that VM-CI outperforms other popular CI tests in terms of either time or sample complexity (or both), which translates to a better performance in structure learning as well.
\end{abstract}
\vspace{-0.5cm}
\section{Introduction}

Causal discovery, the pursuit of uncovering the cause-and-effect relationships governing complex systems, has been the focus of intense research in machine learning, statistics, and various scientific domains over the past few decades. 
This is due to the extensive impact of causal inference, which enables us to make well-informed decisions and policies.

Current approaches for learning causal mechanisms in data can be divided into two categories: \emph{score-based}, e.g., \cite{chickering2002optimal, solus2021consistency, zheng2018dags, zhu2019causal}, and \emph{constraint-based}, e.g., parent-child (PC) algorithm \citep{spirtes2000causation} and grow-shrink (GS) algorithm \citep{margaritis1999bayesian}. 
Score-based approaches place restrictions on the functional causal model and/or the distribution of data. 
As a consequence, these methods can struggle to identify an accurate causal graph when dealing with complex relationships between variables or in the presence of hidden variables. Constraint-based methods often do not rely on the aforementioned assumptions and directly test for conditional independence (CI) relations between pairs of variables to determine causal connections. 

Theoretical performance guarantees of constraint-based discovery algorithms in the literature nearly always hinge on the availability of a perfect CI oracle, which determines whether two random variables are conditionally independent. 
In practice, this oracle is substituted with a statistical conditional independence test, which assesses independence using a limited number of observed data points. 
Hence, to ensure the reliability and applicability of constraint-based causal discovery methods, it is imperative to establish robust sample complexity guarantees. Sample complexity of a causal discovery algorithm is the minimum number of data samples needed to infer the causal graph accurately at a given confidence level. 

Unlike unconditional independence testing, conditional independence is not a testable hypothesis without further assumptions on the distribution. \cite{shah2020} proved this fundamental hardness result by showing that if $(X,Y,Z)$ has an absolutely continuous distribution with respect to the Lebesgue measure and a given CI test has a level less than $\alpha$, there is no alternative under which the test has a power greater than $\alpha$. \cite{neykov21} show that minimax optimal bounds can be obtained when defining an alternative by discarding distributions that are ``$\eps-$close'' to the null hypothesis. 

It is noteworthy that CI tests are well understood for discrete variables; see \cite{canonne18}. Another solved case is that of linear models with Gaussian noise, where conditional independence is equivalent to zero partial correlation, which is easy to test.

In this paper, we derive sample complexity guarantees for CI tests for continuous distributions. Specifically, we design conditional independence tests built upon estimating conditional mutual information, a measure of conditional dependence between variables. The  mutual information $I(X;Y \cond Z)$ between two random variables $X,Y$ conditioned on $Z$ is given by:
\vspace{-0.25cm}
\begin{multline*}
    \Scale[0.99]{\iiint \log\left( \frac{p_Z(z) p_{X,Y,Z}(x,y,z)}{p_{X,Z}(x,z) p_{Y,Z}(y,z)}\right) p_{X,Y,Z}(x,y,z) \,  \mathrm{d}x \, \mathrm{d}y \,  \mathrm{d}z \, .}
\end{multline*} Furthermore,  \begin{flalign}\label{eq:CMI_entropy}
    & I(X;Y \cond Z) = I(X;Y,Z) - I(X;Z) \nonumber \\
    & = H(X,Z) + H(Y,Z) - H(Z) - H(X,Y,Z) \, . 
\end{flalign}
Hence, estimating $I(X;Y \cond Z)$ reduces to estimating entropy,  the approach we shall take. 
Specifically, we use the \emph{Von Mises estimator} $\hat{H}_{\mathrm{vm}}$, defined in Section \ref{sec:von_mises}, which has theoretical and practical advantages. First, this estimator is easy to compute by the usual trick of replacing the integration step with a Monte Carlo type summation. Second, when combined with a (nonparametric) kernel estimate of the density, under smoothness\footnote{The appropriate notion of smoothness, $\beta-$Hölder smoothness, will be introduced in Definition \ref{def:holder}.} assumptions on the joint, $\hat{H}_{\mathrm{vm}}$ converges at the parametric rate $O(n^{-1/2})$, and hence escapes the curse of dimensionality. Finally, it is computationally efficient: its time complexity is linear in the dimension and quadratic in the number of samples. 

\vspace{-0.25cm}
\paragraph{Contributions}
Our main contributions are as follows:
\vspace{-0.25cm}
\begin{itemize}
    \item[$(i)$] We establish an exponential concentration inequality for $\hat{H}_{\mathrm{vm}}$ when the joint density is learned via kernel density estimation (KDE). This 
    allows for deriving a tighter sample complexity bound for $\hat{H}_{\mathrm{vm}}$ compared to those obtained by a standard appeal to Markov's inequality. 
    
    \item[$(ii)$] We define a test for conditional independence, VM-CI, based on Von Mises estimators and establish its sample complexity when discriminating the null hypothesis $H_0$ of conditional independence from an alternative of the form $H_1 := I(X;Y \cond \bZ) > I_{\min}$, with a given level of confidence $1-\alpha$. These results are robust for \emph{all} sufficiently smooth, compactly supported distributions with positive lower bounds.

    \item[$(iii)$] We show that the established sample complexity guarantees of VM-CI yield sample complexity bounds for any constraint-based causal discovery algorithm under mild smoothness assumptions. As an example, we present these bounds for two popular methods, PC and GS. To the best of our knowledge, these are the first sample complexity guarantees for causal discovery algorithms in the continuous case. 
\end{itemize}
\vspace{-0.25cm}
\paragraph{Outline of the paper}
The kernel density estimator, as well as plug-in and Von Mises entropy estimators, are defined in Section \ref{sec:background}. In Section \ref{sec:theory}, we establish the exponential concentration properties of $\hat{H}_{\mathrm{vm}}$, define a CI test based on the former, and derive its error rates. Section \ref{sec:causal} is dedicated to causal discovery, where we derive sample complexity guarantees for PC and GS algorithms. Numerical experiments are presented in Section \ref{sec:exp}. Proofs of our theorems and corollaries are deferred to Appendix \ref{app:proofs}, and further details on numerical experiments can be found in Appendix \ref{app:experiments}.
\vspace{-0.3cm}
\section*{Related work}
\vspace{-0.3cm}
\paragraph{Conditional independence testing for continuous variables}
In the past decade, several methods for CI testing for continuous variables have been developed. One approach (see, e.g., \cite{huang10}) is to discretize the conditioning set $Z$ to a set of bins and perform simple independence tests in each bin. This strategy suffers from the curse of dimensionality, i.e., as the dimension of $Z$ grows, the number of required samples increases drastically. 

Another range of approaches is based on kernel methods. These procedures are comprised of two steps. In the first step,  $X$ and $Y$ are separately regressed on $Z$ via kernel ridge regression \citep{zhang2012kernel}. In the second step, the independence of the residuals is tested. This is often done using the Hilbert-Schmidt independence criterion (HSIC, \cite{gretton05}) or variants of it \citep{zhang2012kernel}. Recently, a so-called generalized covariance measure was used in \cite{shah2020} to test the independence of the residuals.
Theoretical guarantees for the second step, i.e., HSIC (as well as its most recent variants such as Nyström based independence criterion, see \cite{kalinke23}), are now well understood. The standard parametric rate $O(n^{-1/2})$ can be achieved as long as appropriate conditions on the decay rate of the eigenvalues of the corresponding covariance operator are satisfied. 
For the first step, namely kernel ridge regression, as stated in \cite{shah2020}, achieving the parametric rate requires the function $f: z \mapsto \dE[X \cond Z = z]$ to be $\beta-$smooth (say, $\beta-$Hölder) with $\beta>d/2$. The aforementioned approaches suffer from two main drawbacks. The first one is the time complexity of kernel ridge regression: it involves inverting a $n \times n$ matrix, which in general, takes $O(n^3)$ operations. This cubic time complexity prevents the use of the method on large datasets, as we illustrate in Section \ref{sec:exp}. The time complexity of our proposed method is $O(d n^2)$ (see Remark \ref{rem:time_complexity}), which significantly improves on the former. The second drawback is theoretical: to the best of our knowledge, only the rate of convergence under the assumptions listed earlier is known, but for instance, no exponential concentration is established. As a consequence, with existing results, sample complexity guarantees for these methods are not as tight as ours. 

CI testing can also rely on estimating conditional mutual information: the works by \cite{liu2012} and \cite{singh2016} are the most relevant to our study. In both articles, the authors consider kernel density estimate of the joint density, in dimension $d=2$ in the latter and $d \geq 2$ in the former, and prove an exponential concentration for the \emph{plug-in} estimator of entropy. In \cite{poczos11a}, the consistency of a plug-in estimator for R\'enyi divergences using a $k$-nearest neighbors (KNN) estimate of the density was studied. To the best of our knowledge, the convergence rate of this KNN-based estimator is not known. Instead, here we consider the Von Mises estimator combined with a Kernel density estimate (KDE), which is both easier and efficient to compute, and most importantly, converges with better rates than the plug-in estimator\footnote{As a consequence, our smoothness assumption required to obtain the parametric rate $O(n^{-1/2})$ -- the best we can hope for -- is weaker than those in \cite{singh2016} ($\beta > d/2$ versus $\beta > d$).}. Empirically, as we shall see in Section \ref{sec:exp}, KNN-based estimators converge more slowly than the estimator using KDE. Another drawback of KNN is that it is not clear how to tune $k$ in practice, while KDE hyperparameters can be tuned by cross-validation, see e.g., \cite{wassermann23course}. Convergence properties, asymptotic normality, and rates for Von Mises estimators were studied in \cite{kandasamy15}. Our work complements these results by showing an exponential concentration inequality. 

\cite{belghazi2018mutual} proposed the mutual information neural network estimator (MINE) for estimating mutual information between two continuous random variables. 
They rewrite the mutual information using the dual representation of KL divergence \cite{donsker83}, which allows us to formulate the estimation problem as a function optimization. They consider a family of functions parameterized by a deep neural network and solve the optimization using stochastic gradient descent.
They derive a sample complexity bound for an estimator which approximates the true mutual information with $\eps$ error. The bound scales as $\Omega(d \log d /\eps^2)$ and could be applied directly to derive sample complexity bounds for causal discovery algorithms. However, such bounds are overly dependent on dimension $d$. In practice, the estimate requires over  $2 \times 10^6$ to begin to converge, which far exceeds the number of samples we require (see Section \ref{sec:exp}).

A recent work \cite{akbari2023} studies a different approach based on optimal transport (OT). The idea is first to learn a parametric lower triangular monotone map between the unknown joint distribution $p$ and a reference distribution $q$, typically a standard isotropic Gaussian. Once this map is learned, they can estimate the joint distribution $p$ and recover the conditional independence relationships. Although this method appears to be of practical interest, no theoretical guarantees are available, e.g. its consistency is not proved.

\vspace{-.3cm}
\paragraph{Sample complexity in causal discovery}
Sample complexity results for causal discovery are few, and guarantees are only known for discrete variables. \cite{wadhwa2021sample} established the sample complexity of two causal discovery algorithms: inferred causation (IC) and PC, using the CI test introduced by \cite{canonne18}.
This CI test is designed for testing the conditional independence for discrete distributions $p$, namely testing $H_0 := X \indep Y \cond \bZ$ vs $H_1 := \sup_{q} \mathrm{TV}(p,q) > \eps$, where $\mathrm{TV}$ is the total variation distance, $\eps>0$. The supremum is over discrete probability mass functions such that $X \indep_q Y \cond \bZ$. They showed that the output of the CI test is correct with a probability of at least 2/3.
In general, testing causal directions requires additional assumptions or information. In the bivariate discrete case, \cite{acharya2023sample} recently established the sample complexity of distinguishing cause from effect when interventional data is available. They obtain a sample complexity which depends on the domain size and characterize the trade-off between the number of observational and interventional samples.

\vspace{-0.3cm}
\section{Background on kernel density estimation and entropy estimation}\label{sec:background}
\vspace{-0.2cm}
We begin by presenting some definitions and notations that appear throughout the paper. 
Notations $f=o(g), f=O(g)$ and $f=\Theta(g)$ refer to standard Landau notations. The norm $\| \cdot \|_1$ denotes the $L^1$ norm of a vector in $\dR^d$. We assume the $d-$dimensional vector $\bX$ takes values in $\cX$, a compact subset of $\dR^d$. Given a tuple $\bs=(s_1, \ldots, s_d)$ of non negative integers, we define $|\bs|:=\sum_{i=1}^{d} s_j$, ${\bx}^{\bs} := x_1^{s_1} \cdots x_d^{s_d}$, $\bs! := s_1! \cdots s_d!$, and $D^{\bs}$ denotes the operator
    $D^{\bs} := \frac{\partial^{|\bs|}}{\partial^{s_1}x_1 \ldots \partial^{s_d}x_d} \, $.
\begin{definition}[Hölder class, see e.g. \cite{tsybakov2008}, Definition 1.2]\label{def:holder}
    For $L>0$ and a positive integer $\beta$, $f: \cX \subseteq \dR^d \to \dR$ belongs to the \emph{Hölder class} $\Sigma(\beta,L)$ on $\cX$ if $f$ is $ \beta $ times differentiable, and if  for $\bs = (s_1, \ldots, s_d)$ such that $|\bs|=\beta$, $D^{\bs} f$ is bounded by $L$, uniformly in $\bs$ and $\bx$, that is
        $\sup_{\bs:|\bs|=\beta} \sup_{\bx \in \cX} |D^{\bs} f(\bx)| \leq L  \, .$
    $f$ is said to be \emph{$\beta$-Hölder smooth} if $f \in \Sigma(\beta,L)$ for some $L>0$.
\end{definition}

For any $k$ times differentiable function $g$ on $\cX \subseteq \dR^d$, and $\ba \in \cX$, we denote by $g_{k,\ba}$ the truncated degree $k$ Taylor expansion of $g$ at $\ba$, i.e.
\begin{equation*}
    g_{k,\ba}(\bx) := \sum_{s : |s| \leq k} \frac{D^{\bs}g(\ba)}{{\bs}!}(\bx-\ba)^{\bs} \, .
\end{equation*}
\subsection{Plug-in versus Von Mises estimator for entropy}\label{sec:von_mises}

Assume we have access to $n$ samples $(\bx^{(i)})_{1 \leq i \leq n} = ((x^{(i)}_1, \ldots, x^{(i)}_d))_{1 \leq i \leq n} $ of a $d-$dimensional random vector  $\bX=(X_1, \ldots, X_d)$, with density $p$ with a compact support $\cX$ in $\dR^d$. We seek to estimate the joint entropy 
\begin{equation}\label{eq:joint_entropy}
    H(p) := H(X_1, \ldots, X_d) = - \int_{\cX} p(\bx) \log p(\bx) \, \bdx  \, .
\end{equation} The \emph{plug-in} estimator of $H$ is given by 
\begin{equation}\label{eq:plug_in}
    \what{H}_{\mathrm{plug-in}} := - \int_{\cX} \what{p}(\bx) \log \what{p}(\bx) \, \bdx  ,
    \vspace{-0.2cm}
\end{equation} 
where $\what{p}$ is an estimate of the joint probability density. As discussed in the related work, this estimator was studied by \cite{liu2012} and \cite{singh2016}. In practice, computing the numerical approximation of the integral in \eqref{eq:plug_in} is costly when dimension $d$ increases. Herein, we study the \emph{Von Mises} estimator defined as follows:
\begin{equation}\label{eq:von_mises_estimator}
    \what{H}_{\mathrm{vm}} := - \frac{2}{n} \sum_{i=n/2+1}^{n} \log \what{p}(\bx^{(i)}) \, .
    \vspace{-0.2cm}
\end{equation}
To estimate the entropy, the data is split into two parts. The first part is used to estimate the density $\what{p}_h$, and the second half is used to estimate $\what{H}_{\mathrm{vm}}$ using $\what{p}_h$ according to \eqref{eq:von_mises_estimator}.
Note that using Taylor expansion of $p \mapsto -p \log p$ around $\what{p}$ in \eqref{eq:joint_entropy} results in\footnote{This is up to justifying swapping the $O$ and the integral, which will be done later in the proof of Theorem \ref{th:conc_entropy}.}:
\begin{multline}\label{eq:dvlp_von_mises}
    \Scale[0.99]{H(p)= H(\what{p}) - \int_{\cX}(\log \what{p}(\bx) + 1)(p(\bx)-\what{p}(\bx))  \, \bdx} \\
    \Scale[0.99]{\quad + O\left( \int_{\cX} (p(\bx)-\what{p}(\bx))^2  \, \bdx \right)} \\
    \Scale[0.99]{= - \int_{\cX} p(\bx) \log \what{p}(\bx) \, \bdx + O\left( \int_{\cX} (p(\bx)-\what{p}(\bx))^2 \, \bdx \right) ,}
\end{multline}
since $\int_{\cX} p(\bx) \, \bdx  = \int_{\cX} \what{p}(\bx) \, \bdx =1$. This motivates the estimation of $H(p)$ with $- \int_{\cX} p(\bx) \log \what{p}(\bx) \, \bdx $. $\what{H}_{\mathrm{vm}}$ in \eqref{eq:von_mises_estimator} is derived by replacing the integral with the Monte Carlo sum.
Expansion \eqref{eq:dvlp_von_mises} is often referred to as the \emph{Von Mises expansion}, see \cite{krishnamurthy2014}.

As discussed earlier, to estimate the entropy $H$, we need to estimate the joint density $p$. We discuss an approach based on Kernel density estimation in the next section. Please refer to \cite{tsybakov2008} for more details.

\subsection{Kernel density estimation} \label{sec:kde} 
\vspace{-.3cm}
Multivariate kernel density estimation (KDE) provides an estimate of the density $p$ of the following form. For all $\bx = (x_1, \ldots, x_d)$ in $\cX$,
\begin{equation}\label{eq:def_ph}
    \what{p}_h(\bx) := \frac{2}{n} \sum_{i=1}^{n/2} \frac{1}{h^d} K_d \left( \frac{\bx^{(i)} - \bx}{h}\right) , 
\end{equation} where $h := h(n)>0$ is the \emph{bandwidth} and $K_d: \dR^d \to \dR$ is a \emph{kernel}, satisfying $\int K_d(\bx) \, \bdx = 1$ to ensure that $\int_{\cX} \what{p}_h(\bx) \, \bdx =1$. Recall that we use the first half of the samples $(\bx^{(i)})_{1 \leq m \leq n/2}$ to compute $\what{p}_h$.

The choice of $K_d$ is generally very open. However, when approximating smooth densities, kernels of order $\ell >0$ are very useful. We define them below. 

\begin{definition}[Kernels of given order]\label{def:kernel_order}
    Let $\ell$ be a positive integer. We say that a kernel $K_d: \dR^d \to \dR$ is a \emph{kernel of order $\ell$} if $\bx \mapsto {\bx}^{\bs} K(\bx)$ is integrable for all $|\bs| \leq \ell$ and 
    \begin{equation*}
        \int  K(\bx) \mathrm{d}\bx = 1 \mbox{ and } \int {\bx}^{\bs}  K(\bx) \mathrm{d}\bx = 0 \mbox{ for } |\bs| = 1, \ldots, \ell \, .
    \end{equation*}
    In particular, a kernel of order $\ell$ is orthogonal to any polynomial of degree $\leq \ell$ with no constant term.
\end{definition}  
\vspace{-0.3cm}
\paragraph{Product kernels}
In our practical implementations, we will consider product kernels of the form 
\begin{equation*}
    K_d(\bx) = \otimes_d K(\bx) := K(x_1) \cdot K(x_2) \cdots  K(x_d), 
\end{equation*} where $K$ is a one-dimensional kernel satisfying $\int K(u) \, \mathrm{d}u = 1$. We hence have
\begin{equation*}
    \what{p}_h(\bx) := \frac{2}{n} \sum_{i=1}^{n/2} \frac{1}{h^d} K \left( \frac{x_1^{(i)} - x_1}{h}\right) \cdots K \left( \frac{x^{(i)}_d - x_d}{h}\right) \, .
\end{equation*}
    Note that in view of Definition \ref{def:kernel_order}, if $K$ is of order $\ell > 0$ then $\otimes_d K(\bx)$ is also of order $\ell$.
\vspace{-0.3cm}
\paragraph{Legendre kernels}
 \cite{tsybakov2008} (Section 1.2.2) provides a method to build a one-dimensional kernel supported on $[-1,1]$ of any given order $\beta$ as follows. Let $\{ \phi_m \}_{m \geq 0}$ be the orthonormal basis of Legendre polynomials $L^2([-1,1],\mathrm{d}x)$ defined by 
\begin{equation}\label{eq:legendre}
    \phi_m(x) := \sqrt{\frac{2m+1}{2}} \frac{1}{2^m m!} \frac{\mathrm{d}^m}{\mathrm{d}x^m} [(x^2-1)^m] 
\end{equation} for all $m \geq 0$, with $\phi_0(x) = \frac{1}{\sqrt{2}}$ by convention. Then for $\beta>0$, the kernel $K_\beta$ defined by 
\begin{equation}\label{eq:Kbeta}
    K_{\beta}(x) := \sum_{m = 0}^{\beta} \phi_m(0) \phi_m(x) \mathbbm{1}_{|x| \leq 1}
\end{equation} 
is of order\footnote{Note that by symmetry of Legendre polynomials, $\phi_{2m+1}(0) = 0$ for all $m \geq 0$, hence $K_{2 \ell} = K_{2 \ell + 1}$. We will hence often consider $\beta$ to be odd so that $K_{\beta}$ is exactly of order $\beta$ and not of order $\beta+1$.} $\beta$. We will henceforth refer to kernels $K_{\beta}$ as Legendre kernels. 

\section{Exponential concentration for entropy estimation}\label{sec:theory}
\vspace{-0.3cm}
In this section, we present one of our main results, the exponential concentration for our MI estimator. For pedagogical reasons, we begin with presenting exponential concentration for the KDE estimator, a result known in the literature. We then proceed to establish an exponential concentration for the Von Mises estimator of entropy and, finally, the  MI estimator. A tight upper bound on the error rate of our VM-CI test follows as a corollary. 
We recall that the kernel density estimator $\what{p}_h$ is defined in \eqref{eq:def_ph}.

\subsection{Exponential concentration for multivariate kernel density estimation}

To obtain the exponential concentration of $\what{p}_h$, we need to first establish a few technical conditions on kernel $K_d$.

\begin{assumption}[Assumptions on the kernel $K_d$]\label{ass:kernel}
\, 
\vspace{-0.35cm}
\begin{itemize}
    \item[$(1a)$] $K_d$ is uniformly upper bounded by some $\kappa>0$,

    \item[$(1b)$] $K_d$ is of order $\beta$ (see Definition \ref{def:kernel_order}),

    \item[$(1c)$] The class of functions
    $$ \cF := \left\{ K_d \left( \frac{\bx-\cdot}{h}\right), \bx \in \dR^d, h>0 \right\}$$ satisfies 
    $ \sup_Q N(\cF,L^2(Q),\eps \| F \|_{L^2(Q)}) \leq \left( \frac{A}{\eps}\right)^v,$
\end{itemize}
where $A$ and $v$ are for two positive numbers, $N(T,d,\eps)$ denotes the $\eps$-covering number (see, e.g. \cite{lafferty08}) of the metric space $(T,d)$, $F$ is the envelope function of $\cF$ (i.e. $F(\bx):= \sup_{f \in \cF} |f(\bx)|$),  and the supremum is taken over the set of all probability measures on $\dR^d$. 
The quantity $v$ is called the $VC$ dimension of $\cF$.
\end{assumption}

Assumption $(1c)$ appears in \cite{gine02, rinaldo2010} and is at the heart of the exponential inequality obtained in \cite{liu2012}. This assumption is known to hold for a large class of kernels \cite{vandervaart1996, nolan87}, such as compactly supported polynomial kernels and Gaussian kernels\footnote{Assumption $(1c)$ also holds in the following examples: if $K_d(\bx) = \phi(T(\bx))$, where $T$ is a polynomial in $\dR^d$ and $\phi$ a bounded real function of bounded variation if the graph of $K_d$ is a pyramid (truncated or not); or if $K_d = \mathbbm{1}_{I_1 \times \cdots \times I_d}$ where $I_1, \ldots, I_d$ are closed intervals of $\dR$ \citep{vandervaart1996, nolan87}.}.

\begin{remark}\label{rem:example}
Kernel $K_{\beta}$ defined in \eqref{eq:Kbeta} satisfies Assumption \ref{ass:kernel}. Therefore, product kernel $K_d := \otimes_d K_{\beta}$ inherits the same property. 
\end{remark}

\begin{theorem}[Exponential concentration of $\| p -\what{p}_h\|_{\infty}$]\label{thm:conc_sup}
    Assume that $p$ belongs to the \emph{Hölder class} $\Sigma(\beta,L)$  on $\cX$ for some $\beta,L>0$ and that $K_d$ satisfies Assumption \ref{ass:kernel}.
   Let $h = h_n = \Theta(n^{-\frac{1}{2\beta + d}})$. Then, there exist $C_1, C_2, \eps_0>0$ and $n_0 \geq 0$ such that for all $ n^{-\frac{\beta}{2\beta + d}} (\log n)^{1/2} \leq \eps_n \leq \eps_0$:
    \begin{equation*}
       \forall n \geq n_0, \;  \dP\left(\|p-\what{p}_h\|_{\infty} > \eps_n \right) \leq C_1 \exp(- C_2 n^{\frac{2 \beta}{2 \beta + d}} \eps_n^2) \, .
    \end{equation*}
\end{theorem}
The proof of this result, which appears in Appendix \ref{app:proofs} for the sake of completeness, follows from standard bias analysis and results in \cite{rinaldo2010}. 

\vspace{-0.3cm}
\subsection{Exponential concentration for entropy estimation}
Before stating our result on the exponential concentration of estimator $\what{H}_{\mathrm{vm}}$ -- which we recall is defined in \eqref{eq:von_mises_estimator} -- we describe the conditions which density $p$ must satisfy.

\begin{assumption}[Assumptions on the density $p$]\label{ass:density}
\, 
\vspace{-0.35cm}
\begin{itemize}
    \item[$(2a)$] The support of $p$, $\cX$, is a compact set  in $\dR^d$,

    \item[$(2b)$] $p$ is lower-bounded on $\cX$ by some $p_{\min} >0$,

    \item[$(2c)$] $p$ belongs to \emph{Hölder class} $\Sigma(\beta,L)$ for some $L>0$.
\end{itemize}
\end{assumption} 
\begin{remark}[Positivity of $\what{p}_h$]\label{rem:positivity_ph}
    The kernel $K_d$ can take negative values\footnote{Note e.g. that any kernel or order $\beta \geq 3$ needs to take negative values by definition.}, as does $\what{p}_h$.
    Hence, $\log \what{p}_h$ will not be defined in general, which poses issues in the definition of  $\what{H}_{\mathrm{vm}}$.
    As proved by Giné and Guillou (see Theorem 2.3 in \cite{gine02}), Theorem \ref{thm:conc_sup} together with \eqref{eq:proof:thm:conc_sup3} and an application of Borel-Cantelli Lemma shows that almost surely, $n^{\frac{\beta}{2\beta + d}}(\log n)^{-1/2} \| p -\what{p}_h\|_{\infty}$ converges to some bounded random variable $C$. As a result, if $p$ satisfies $(2a)$, then almost surely there exists $n_0$ such that for $n \geq n_0$, $\what{p}_h$ is point-wise positive. In the sequel, indeed, we assume that $n$ is large enough.
\end{remark}


\begin{theorem}[Exponential concentration of  $\what{H}_{\mathrm{vm}}$ in \eqref{eq:von_mises_estimator}]\label{th:conc_entropy}
    Assume that $K_d$ satisfies Assumption \ref{ass:kernel} and that $p$ satisfies Assumption \ref{ass:density}. 
    Let $h = h_n = \Theta(n^{-\frac{1}{2\beta + d}})$. Then, there exist $C_1, C_2, C'_1, C'_2, \eps_0>0$ and $n_0 \geq 0$, such that for all $\eps_n$ such that $\max (n^{-\frac{2\beta}{2\beta + d}} \log n, n^{-1/2}) \leq \eps_n \leq \eps_0$: 
    \vspace{-0.2cm}
    \begin{multline*}
       \forall n \geq n_0, \;  \dP(|\what{H}_{\mathrm{vm}}-H(p)| > \eps_n ) \leq C_1 e^{ - C_2 n^{\frac{2 \beta}{2 \beta + d}} \eps_n}\\
       +C'_1 e^{- C'_2 n^{1/2} \eps_n} \, .   
    \end{multline*}
\end{theorem}
\begin{remark} Note that when $p$ is smooth enough ($\beta > d/2$), $\what{H}_{\mathrm{vm}}$ converges at parametric rate $O(n^{- 1/2})$, the best rate we can hope for. 
\end{remark}
\begin{remark}
    The well-known rate $O(n^{-\min(\frac{1}{2},\frac{2\beta}{2\beta+d})})$ for Von Mises entropy estimation (see \cite{wassermann23course}) is immediate from Theorem \ref{th:conc_entropy}.
   Note that when $d=2$, we retrieve the concentration inequality of \cite{liu2012}. The minimax rates for entropy estimation are known to be slightly better, $O(n^{- \min (\frac{1}{2},\frac{4 \beta}{4 \beta + d})})$, but come at the cost of more complex estimators, requiring higher order corrections in the Von Mises expansion \eqref{eq:dvlp_von_mises}. 
\end{remark}

\subsection{Consequences for error rates of VM-CI}

We start with an immediate corollary of Theorem \ref{th:conc_entropy}, which states a dimension-free exponential concentration bound for conditional mutual information as long as the probability distributions are smooth enough. Given our application of interest, causal discovery, we assume  $X$ and $Y$ are both one-dimensional, but $\bZ$ is of dimension $d_{\bZ}$. In view of \eqref{eq:CMI_entropy}, we can estimate $I(X;Y \cond \bZ)$ by
\begin{flalign*}\label{eq:est_CMI}
    \Scale[0.9]{\what{I}_{\mathrm{vm}} := \what{H}_{\mathrm{vm}}(X,\bZ) + \what{H}_{\mathrm{vm}}(Y,\bZ) - \what{H}_{\mathrm{vm}}(\bZ) - \what{H}_{\mathrm{vm}}(X,Y,\bZ) \, ,}
\end{flalign*} where $\what{H}_{\mathrm{vm}}$ is the Von Mises estimator in \eqref{eq:von_mises_estimator}.  

\begin{corollary}[Dimension-free exponential concentration of $\what{I}_{\mathrm{vm}}$]\label{cor:conc_MI}
    Assume that 
    \vspace{-0.35cm}
    \begin{itemize}
        \item joint distributions $p_{X,Y, \bZ}$, $p_{X, \bZ}$, $p_{Y, \bZ}$, and $p_{\bZ}$ satisfy Assumption \ref{ass:density} for some $\beta >0$ such that $\beta > 1 + d_{\bZ}/2$;
        
        \item kernels 
        involved in estimators $\what{H}_{\mathrm{vm}}(X,Y,\bZ)$ (resp. $\what{H}_{\mathrm{vm}}(X,\bZ)$, $\what{H}_{\mathrm{vm}}(Y,\bZ)$ and $\what{H}_{\mathrm{vm}}(\bZ)$) satisfy Assumption \ref{ass:kernel}, with $\beta$ given in the previous bullet. 
        
    \end{itemize} Choose bandwidth $h_n$ as follows:
    \begin{equation*}
        h_n = 
        \left\{
    \begin{array}{ll}
        \Scale[0.9]{\Theta(n^{-\frac{1}{2\beta + 2 + d_{\bZ}}}) \mbox{ for } \what{H}_{\mathrm{vm}}(X,Y,\bZ),} \\
        \Scale[0.9]{\Theta(n^{-\frac{1}{2\beta + 1 + d_{\bZ}}}) \mbox{ for } \what{H}_{\mathrm{vm}}(X,\bZ) \mbox{ and } \what{H}_{\mathrm{vm}}(Y,\bZ),} \\
        \Scale[0.9]{\Theta(n^{-\frac{1}{2\beta + d_{\bZ}}}) \mbox{  for } \what{H}_{\mathrm{vm}}(\bZ),}
    \end{array}
\right.
    \end{equation*} then, there exist $C_1, C_2,\eps_0>0$ and $n_0 \geq 0$ such that for all constant $ 0 < \eps \leq \eps_0$,
    \begin{equation*}
       \Scale[0.93]{\forall n \geq n_0, \;  \dP(|\what{I}_{\mathrm{vm}}- I(X;Y \cond \bZ)| > \eps ) \leq  C_1 \exp(- C_2 n^{1/2} \eps ) \, .}
    \end{equation*}
\end{corollary}
To provide performance guarantees for our CI test, we require the following mild assumption.
\begin{assumption}[Minimum level of dependency]\label{ass:min_level}
There exists $I_{min}>0$ such that 
$X$, $Y$, s and $\mathbf{Z}$ are either conditionally independent (i.e., ${X} \independent {Y} \cond {\mathbf{Z}}$) or $I(X;Y \cond \bZ) > I_{min}$.
\end{assumption}
Under the Assumption \ref{ass:min_level}, we can define the following hypothesis test.
\begin{equation*}
    H_0 := I(X;Y \cond \bZ)=0 \mbox{ vs. } H_1 := I(X;Y \cond \bZ) > I_{\min} \, .
\end{equation*} The test, VM-CI, is defined as follows.
\begin{equation}\label{eq:cor:rate_CItest}
        T_{\mathrm{VM-CI}} := \begin{cases}
            1 \quad \mbox{if } \what{I}_{\mathrm{vm}} > I_{\min} /2, \\
            0 \quad \mbox{elsehow. }
        \end{cases}
    \end{equation}

\begin{corollary}[Error rates for VM-CI]\label{cor:rate_CItest}
Under Assumption \ref{ass:min_level} as well as the assumptions stated in Corollary \ref{cor:conc_MI}, the sum of type one and type two errors for $T_{\mathrm{VM-CI}}$ is bounded by $O\left( \exp(-c n^{1/2} I_{\min}) \right)$ for some $c>0$. Hence, in order to achieve a confidence level $1-\alpha \in [0,1)$ it suffices that
    $ n \geq \Omega\left( \frac{1}{I_{\min}^2} \log^2\left( \frac{1}{\alpha}\right) \right) \, .$
\end{corollary}

\begin{remark}[Time complexity of VM-CI] \label{rem:time_complexity}
Assume that each evaluation of $K_d$ is done in $O(d)$. Then, each appeal to $\what{p}_h$ takes $O(d n)$ operations. Hence, $\what{H}_{\mathrm{vm}}$, $\what{I}_{\mathrm{vm}}$, and VM-CI can be computed in $O(d n^2)$.
\end{remark}

\section{Application: sample complexity guarantees for causal discovery}\label{sec:causal}
In this section, we present a brief background on causal discovery and review two classic causal algorithms, PC and GS, before deriving their sample complexity when using the VM-CI test. Note that under appropriate assumptions, these sample complexities are optimal since they inherit the parametric convergence rate of VM-CI, which is the best we can hope for.

\subsection{Background on causal discovery}
A directed acyclic graph (DAG) is defined as $\cG = (\bX, E)$, where $\bX = \{ X_1, \ldots, X_m\}$ (resp. $E \subseteq \bX \times \bX$) denotes the set of vertices (resp. directed edges) of $\cG$, such that $\cG$ contains no directed cycle. Each vertex $X_k \in \bX$ represents a random variable. 
Vertices $X, Y \in \bX$ are called neighbors in $\cG$ if $(X, Y)$ or $(Y, X)$ belongs to $E$. We denote the set of neighbors of $X$ $\cG$ by $N_\mathcal{G}(X)$.
Causal discovery (a.k.a structure learning) is the task of learning the causal graph $\cG$ from $n$ i.i.d. samples drawn from the joint distribution $p$, commonly referred to as the observational distribution. 

We assume that $\cG$ and $p$ satisfy \emph{Markov} and \emph{faithfulness} properties, which state that conditional independence relationships in $p$ correspond to so-called d-separation (a graphical condition) in  $\cG$. We refer the reader to \cite{pearl2009causality} for definitions and further discussion on this topic. Two DAGs satisfying \emph{Markov} and \emph{faithfulness} properties are \emph{Markov equivalent} if they have the same set of d-separations (i.e., encode the same set of conditional independence). The equivalence class of a DAG $\cG$ is called the Markov equivalence class (MEC) of $\cG$.
It is well-known \citep{spirtes2000causation, pearl2009causality} that without further assumptions, we can only learn the underlying causal DAG up to its Markov equivalence from the observational data alone. 

\subsection{Parent-child (PC) algorithm \citep{spirtes2000causation}}
PC  begins with a complete, undirected graph $\cC$ on the vertex set $\bX$.
Starting from $\ell=0$, the algorithm considers pairs of variables $X$ and $Y$ adjacent in $\cC$  such that $|N_{\mathcal{C}}(X) \setminus {Y}| \geq \ell$.
For all $\bZ \subseteq N_\mathcal{C}(X) \setminus {Y}$ such that $|\bZ|=\ell$, PC iteratively tests $X \indep Y \cond \bZ$. If the conditional independence holds for a subset $\bZ$, the edge $\{X,Y\}$ is removed in $\cC$. 
After step $\ell = \Delta$,  the maximum degree in $\cG$, the process terminates. The last step of the algorithm consists of orienting the edges in $\cC$, leveraging the information acquired in the previous phase as well as applying so-called Meek rules (see \cite{meek1995causal}). If all CI tests outputs are correct, the final graph $\cC$ is the essential graph\footnote{The essential graph of $\cG$ represents the Markov equivalence class of $\cG$. Namely, it has the same skeleton and v-structures (see \cite{pearl2009causality}).} of $\cG$. Furthermore, recall that $m$ is the number of nodes in $\cG$.

\begin{theorem}[Sample complexity of PC]\label{th:PC} 
Assume that all CI tests involving $(X,Y,\bZ)$ with $X, Y \in \bX$ and  $\bZ \subseteq \bX, |\bZ| \leq \Delta$\footnote{ $\Delta$ is the maximum degree or an upper bound on the maximum degree in $\cG$.}, Assumptions \ref{ass:kernel} (on the kernel), \ref{ass:density} (on the joint), and \ref{ass:min_level} (on the minimum level of dependency) are satisfied. 
Let $\alpha >0$. PC using VM-CI tests with threshold $I_{min}/2$ recovers the MEC of $\cG$ with probability $\geq 1-\alpha$, as long as
$ n \geq \Omega\left(\left( \frac{\Delta+1}{I_{min}} \log (m/\alpha) \right)^2\right) \, . $
\end{theorem}

\begin{remark}
    The sample complexity result of  Theorem \ref{th:PC} results from the exponential concentration derived in Theorem \ref{th:conc_entropy}. The previously known rate \citep{wassermann23course} $\dE[|\what{I}_{\mathrm{vm}}- I|] \leq Cn^{-1/2}$ and applying Markov's inequality yields the much looser\footnote{Note that the sample complexity of PC obtained when MINE \citep{belghazi2018mutual} is used to estimate the mutual information scales as $\Omega\left(\left( \frac{(m/2+\Delta+1)}{I_{min}} \log (m/\alpha) \right)^2\right)$, which already improves over Markov inequality but is still much looser than ours.} bound $n \geq \Omega\left(\left( \frac{m^{\Delta+1}}{\alpha I_{min}} \right)^2\right)$. 
\end{remark}
\subsection{Grow-shrink (GS) algorithm}
\begin{definition}[Markov boundary]
    The Markov boundary of a random variable  $X$ in set $\bX$, denoted by $\mathrm{MB}(X)$, is a minimal set $\bS \subseteq \bX \setminus \{X\} $ such that $X \indep \bX \setminus (\bS \cup \{X\}) \cond \bS$.
\end{definition}

The GS algorithm \citep{margaritis1999bayesian} first recovers the Markov boundary of each variable $X \in \bX$, as follows. Starting with $\mathrm{MB}(X)= \varnothing$,
\begin{itemize}
\vspace{-0.3cm}
    \item[1.] (Growing phase) While $\exists \, Y \in \bX \setminus \{X\}$ such that $Y \indep X \cond \mathrm{MB}(X)$, add $Y$ to $\mathrm{MB}(X)$,
    \item[2.] (Shrinking phase) While $\exists \, Y \in \mathrm{MB}(X)$ such that $Y \not\indep X \cond \mathrm{MB}(X) \setminus \{Y\}$, remove $Y$ from $\mathrm{MB}(X)$.
\end{itemize}
\vspace{-0.2cm}
Then, GS recovers the non-oriented graph structure. For every $X \in \bX$ and $Y \in \mathrm{MB}(X)$, a non-oriented edge $\{X,Y\}$ is added if
\begin{itemize}
\vspace{-0.3cm}
    \item[3.] for all $\bS \subseteq \mathbf{T}$, where $\mathbf{T}$ is the set with the smaller cardinality between $\mathrm{MB}(X) \setminus \{Y\}$ and $\mathrm{MB}(Y) \setminus \{X\}$, it holds that ${X} \not\indep {Y} \cond \bS$. 
\end{itemize}
\vspace{-0.2cm}
Finally, every edge $\{ X,Y\}$ is oriented $Y \to X$ if 
\begin{itemize}
\vspace{-0.3cm}
    \item[4.] $\exists \, Z \in N(X) \setminus (N(Y) \cup \{Y\}) $ such that for all $\bS \subseteq \mathbf{W}$, where $\mathbf{W}$ is the set with the smaller cardinality between $\mathrm{MB}(Y) \setminus \{X, Z\}$ and $\mathrm{MB}(Z) \setminus \{X, Y\}$, it holds that $Y \not\indep Z \cond {\mathbf{S} \cup \{X\}}$.  
\end{itemize}
The last step of GS is the same as in PC, namely, it applies the Meek rules.

\begin{theorem}[Sample complexity of GS]\label{th:GS}
Assume that $\max_{X\in \bX} |\mathrm{MB}(X)| \leq  \Gamma$ and that for all CI test involving $(X,Y,\bZ)$ with $X, Y \in \bX, \bZ \subseteq \bX, |\bZ| \leq \Gamma$, Assumptions \ref{ass:kernel} (on the kernel), \ref{ass:density} (on the joint), and \ref{ass:min_level} (on the minimum level of dependency) are satisfied.
Then, for GS using VM-CI tests with threshold $I_{min}/2$ recovers the MEC of $\cG$ with probability $\geq 1-\alpha$, as long as
$n \geq \Omega\left( \frac{1}{I_{min}}\log \left(\frac{m^2+m\Gamma^2 2^\Gamma}{\alpha} \right)^2\right) \, . $
\end{theorem}

\section{Numerical experiments}\label{sec:exp}
\vspace{-0.2cm}
\subsection{Experiments for single conditional independence test}
We compared VM-CI to other CI tests discussed in the related work,  including the KNN-based estimator \citep{poczos11a}, MINE \citep{belghazi2018mutual}, the HSIC-based CI test \footnote{Provided in \cite{kalainathan2020causal}.}\citep{zhang2012kernel}, the OT-based method \citep{akbari2023}, and the standard Gaussian partial correlation test.
\begin{figure*}[!ht]
    \centering
    \begin{subfigure}{0.8\textwidth}
        \includegraphics[trim=280 0 280 20, angle=270,width=\linewidth, clip]{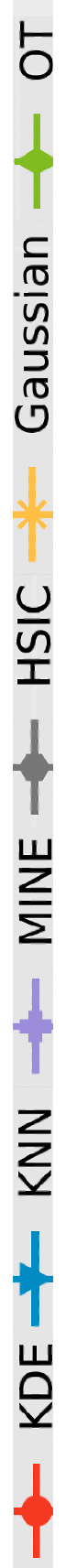}
        \vspace{-.15cm}
    \end{subfigure}
    \hfill
    \begin{subfigure}[t]{0.23\textwidth}
        \includegraphics[width=1.25\linewidth]{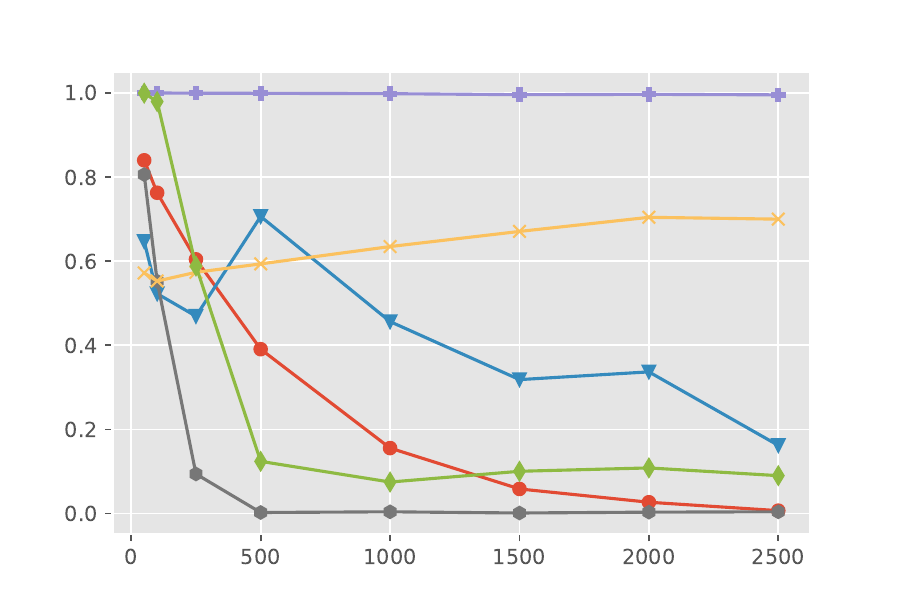}
        \caption{Type I+II errors for power-law data.}
        \label{fig:CI_test_power}
    \end{subfigure}
    \hspace{0.01\textwidth}
    \hfill
    \begin{subfigure}[t]{0.23\textwidth}
        \includegraphics[width=1.25\linewidth]{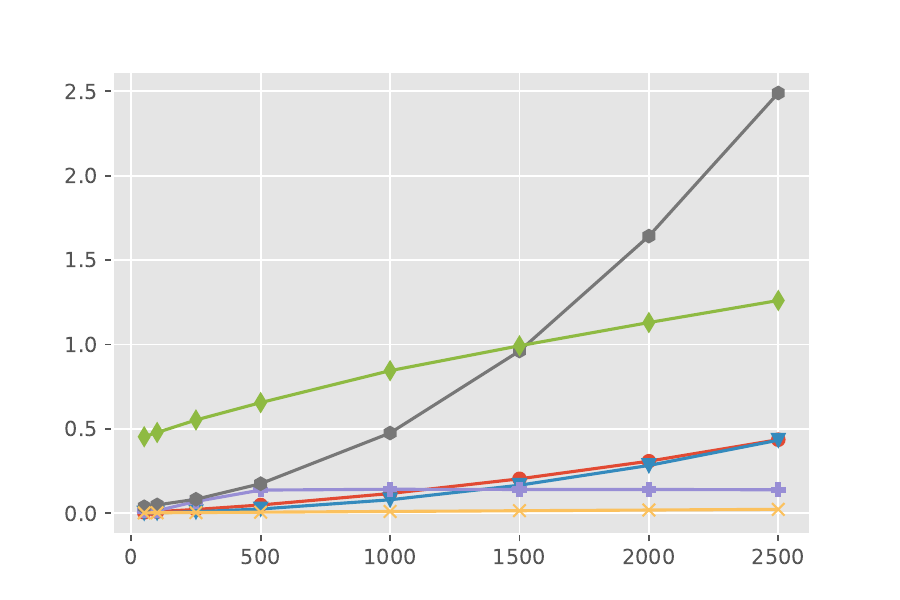}
        \caption{Time spent (s) for each CI test for power-law data.}
        \label{fig:nongaus_time}
    \end{subfigure}
    \hspace{0.01\textwidth}
    \hfill
    \begin{subfigure}[t]{0.23\textwidth}
        \includegraphics[width=1.25\linewidth]{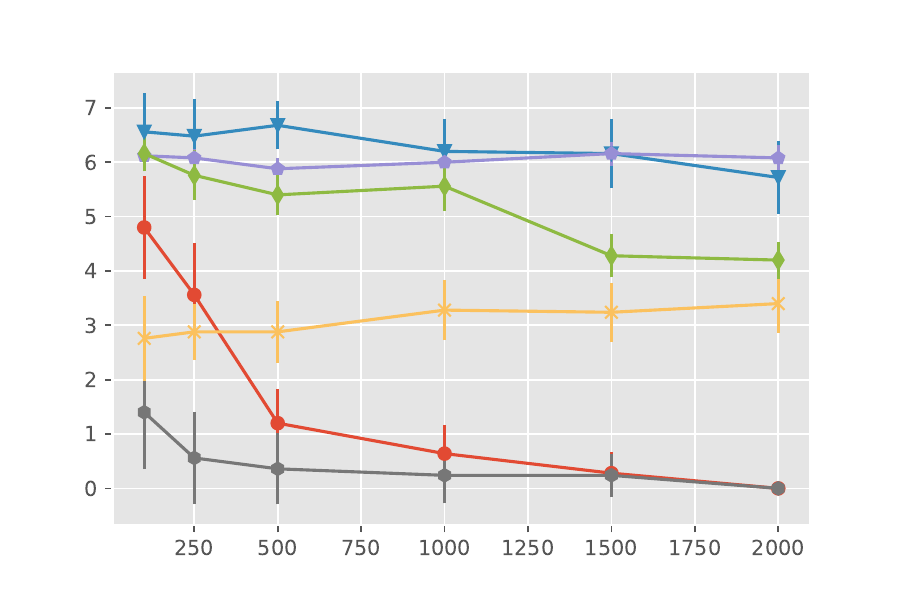}
        \caption{Overall loss of PC algorithm with various CI testers.}
        \label{fig:PC_loss}
    \end{subfigure}
    \hspace{0.01\textwidth}
    \hfill
    \begin{subfigure}[t]{0.23\textwidth}
        \includegraphics[width=1.25\linewidth]{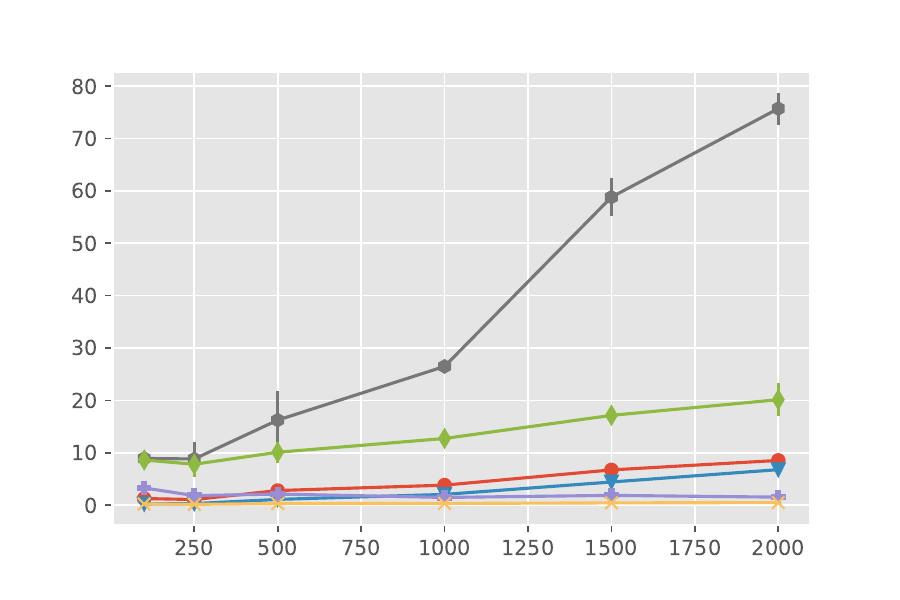}
        \caption{Runtime (s) of PC algorithm with various CI testers.}
        \label{fig:pc_time}
    \end{subfigure}
    \caption{Results of the numerical experiments (on the x-axis: number of samples $n$)}
    \label{fig:results}
\end{figure*}
We conducted the experiments using power-law distributed synthetic data. 
For each value of $n$, we ran $n_{\mathrm{exp}} = 5000$ experiments, the first half with $X \indep Y \cond \bZ$, the second half with $I(X;Y \cond \bZ) > I_{min}$.
The resulting estimated errors (sum of type I and type II errors) are depicted as a function of $n$ in Figures \ref{fig:CI_test_power}. In these figures, our method is denoted as ``KDE''.
More details regarding the generative models and parameters can be found in Appendix \ref{app:experiments}.

These results illustrate that, when dealing with non-Gaussian data, VM-CI outperforms most other methods in terms of type I and type II errors, with the exception of the HSIC-based method \citep{zhang2012kernel}.  However, VM-CI competes favorably with HSIC when the sample size $n$ exceeds $1500$, and it is significantly faster, as demonstrated in Figure \ref{fig:nongaus_time}. Among the methods we explored, the OT-based approach of \citep{akbari2023} comes close to that of VM-CI, although its performance and time complexity fall slightly short of VM-CI.
\vspace{-0.2cm}
\subsection{Experiments for causal discovery algorithms}\label{sec: exp disc}
In addition to CI tests, we ran experiments to assess the performance of PC using VM-CI vs. PC using other CI tests. These experiments were performed on non-Gaussian synthetic data generated according to a Structural Equation Model (SEM) with the causal graph depicted in Figure \ref{fig:true_causal_graph}.

\begin{figure}[ht]
	\centering
	\begin{tikzpicture}[scale=1, every node/.style={scale=0.8}]
	   \tikzset{vertex/.style = {shape=circle,draw,minimum size=1em}}
	   \tikzset{edge/.style = {->,very thick,> = latex,sibling distance=20mm}}
		\node[vertex] (a) at (0,0) {$1$};
		\node[vertex] (b) [right  = 1.05cm of a] {$2$};
		\node[vertex] (c) [below left = 0.8cm of b] {$3$};
		\node[vertex] (e) [below left = 0.8cm of a] {$5$};
		\node[vertex] (d) [above left = 0.8cm of e] {$4$};
		\node[vertex] (f) [below left = 0.8cm of d] {$6$};
		\draw[edge] (a) to (c);
		\draw[edge] (b) to (c);
		\draw[edge] (a) to (e);
		\draw[edge] (d) to (e);
		\draw[edge] (d) to (f);
		\draw[edge] (e) to (f);
	\end{tikzpicture}
\caption{Underlying causal graph in the experiments.}
\label{fig:true_causal_graph}
\end{figure}
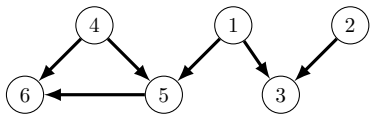
For each value of $n$, we conducted $25$ experiments and depicted the overall loss in Figures \ref{fig:PC_loss}. The overall loss is defined as the total number of missing, extra, and misoriented edges in the resulting graph. Additional details regarding the SEM and its parameters, along with similar experiments for GS, can be found in Appendix \ref{app:experiments}.


VM-CI outperforms all methods except HSIC. Here again, VM-CI is competitive with HSIC when $n \geq 1500$, and it significantly outpaces HSIC in terms of computational efficiency. 
It is noteworthy that, despite being an efficient CI test, the OT-based method exhibits poor performance in our example.
This may be attributed to several factors: $(i)$ the absence of theoretical guarantees such as consistency for the OT method; $(ii)$ the lack of robustness of PC/GS to errors in CI tests\footnote{Indeed, a single error in one of the many CI tests in PC/GS can lead to a drastically different graph.}; and $(iii)$ the strong dependence of the performance of this method on the dimension. The remaining results are consistent with the performance of the CI tests in Figure \ref{fig:CI_test_power}.

To conclude, we emphasize that among the reviewed methods, those that compete favorably with VM-CI either suffer from the lack of theoretical guarantees or have a prohibitive time 
 complexity, or both.

\vspace{-0.2cm}
\section{Conclusion}
\vspace{-0.2cm}
We established an exponential concentration inequality for the nonparametric Von Mises estimator.
Using this estimator, we designed VM-CI to test conditional independence. This test achieves optimal parametric rates under smoothness assumptions and provides a tight upper bound for the error.
This further allowed us to compute the sample complexity of causal discovery algorithms using VM-CI, the first such guarantee for continuous variables.
Our empirical findings show that VM-CI overall outperforms other popular conditional independence tests in terms of time, sample complexity, or both.
The methods that are competitive against VM-CI require excessive time complexity or suffer from a lack of theoretical guarantees.
\bibliographystyle{apalike}
\bibliography{arxiversion}

\clearpage
\appendix
\onecolumn
\section{Proofs}\label{app:proofs}
\subsection{Proof of Theorem \ref{thm:conc_sup}}
\begin{proof}[Proof of Theorem \ref{thm:conc_sup}]
    To prove our result, we use the standard bias-variance decomposition:
    \begin{equation}
        |p(\bx)-\what{p}_h(\bx)| \leq \underbrace{|p(\bx)-p_h(\bx)|}_{\mbox{bias}} + \underbrace{|p_h(\bx)-\what{p}_h(\bx)|}_{\mbox{variance}} ,
    \end{equation} where we recall that $p_h := \dE[\what{p}_h]$. We bound the bias and variance terms separately.

    \underline{\textit{Bounding the bias.}}
    \begin{flalign}\label{eq:proof:thm:conc_sup1}
        p(\bx)-p_h(\bx) & = p(\bx)- \frac{2}{n} \sum_{i=1}^{n/2} \int  \frac{1}{h^d} K_d \left( \frac{\bx^{(i)} - \bx}{h}\right)  p(\bx^{(i)}) \, \bdx^{(i)} \nonumber \\
        & = p(\bx)- \int  \frac{1}{h^d} K_d \left( \frac{\bx' - \bx}{h}\right)  p(\bx') \, \bdx' \nonumber \\
        & \overset{(a)}{=} \int K_d(\by) (p(\bx)-   p(\bx+h\by)) \, \bdy, 
    \end{flalign} where  $(a)$ results from change of variable $\by = (\bx'-\bx)/h$. We now take advantage of the fact that functions in $\Sigma(\beta,L)$ are well approximated by their Taylor expansions. Namely, we have the following classical result:
        
    \begin{lemma}\label{lem:holder_taylor}
    If $g \in \Sigma(\beta,L)$ on $\cX \subseteq \dR^d$, then for all $\ba,\bx \in \cX$,
    \begin{equation}\label{eq:lem:holder_taylor}
    \left| g(\bx) - g_{\beta-1,\ba}(\bx)\right| \leq L\frac{\| \bx-\ba\|_1^\beta}{\beta!} \, .
    \end{equation}
    \end{lemma}
    
    \begin{proof}[Proof of Lemma \ref{lem:holder_taylor}]
        We apply Taylor's theorem at the order $\beta - 1$. There exists $c \in [0,1]$ such that 
        \begin{equation*}
            g(\bx) = \sum_{|s| \leq \beta-1} \frac{D^{\bs}g(\ba)}{{\bs}!}(\bx-\ba)^{\bs} + \sum_{|s| = \beta} \frac{D^{\bs}g(\ba + c(\bx - \ba))}{{\bs}!}(\bx-\ba)^{\bs} 
        \end{equation*} Hence
        \begin{equation*}
            \left| g(\bx) - g_{\beta-1,\ba}(\bx)\right| \leq \sum_{|s| = \beta} L \frac{|\bx-\ba|^{\bs}}{{\bs}!} = L\frac{\| \bx-\ba\|_1^\beta}{\beta!},
        \end{equation*} by the multinomial theorem.
    \end{proof}

    With Lemma \ref{lem:holder_taylor} in mind, \eqref{eq:proof:thm:conc_sup1} becomes
    \begin{flalign}\label{eq:proof:thm:conc_sup2}
        |p(\bx)-p_h(\bx)| & \leq \left| \int K_d(\by) (p(\bx)-   p_{\beta-1,\bx}(\bx+h\by)) \, \bdy \right| \nonumber \\ 
        & \quad \quad \quad +   \int | K_d(\by) (p(\bx+h\by)-   p_{\beta-1,\bx}(\bx+h\by)) | \, \bdy 
    \end{flalign}
    Note that $p(\bx)-   p_{\beta-1,\bx}(\bx+h\by)$ is a polynomial in $\by$, of degree $\leq \beta$, and with no constant term. Since $K_d$ is of order $\beta$, the first term of the RHS of \eqref{eq:proof:thm:conc_sup2} evaluates to $0$. This gives in turn, applying Lemma \ref{lem:holder_taylor}, 
    \begin{equation}\label{eq:proof:thm:conc_sup3}
        |p(\bx)-p_h(\bx)| \leq L h^{\beta} \int | K_d(\by) | \|\by\|_1^{\beta} \, \bdy \leq C h^{\beta} \, ,
    \end{equation} for some constant $C>0$, since $\by \mapsto | K_d(\by) | \|\by\|_1^{\beta}$ is integrable by assumption. Note that the bound \eqref{eq:proof:thm:conc_sup3} is uniform in $\bx \in \cX$.

    \underline{\textit{Bounding the variance.}}
    The variance satisfies an exponential concentration property, thanks to Assumption $(1c)$. We leverage on a result from \cite{rinaldo2010}, obtained by applying some previously established results from \cite{gine02}.

    \begin{proposition}[Proposition 9 in \cite{rinaldo2010}]\label{prop:rinaldo}
        Assume that $K_d$ satisfies $(1a)$ and $(1c)$. Then, for any $D_1>0$ there exists constants $D_2, D_3, \eps_0 >0$ and $n_0 >0$ such that, if $h_n \to 0$, $\frac{h_n^d}{|\log h_n|} \to 0$, and $D_1 \sqrt{\frac{|\log h_n|}{n h_n^d}} \leq \eps_n \leq \eps_0$, then
        \begin{equation}\label{eq:prop:rinaldo}
            \dP\left(\sup_{\bx \in \cX} |p_h(\bx)-\what{p}_h(\bx)| > \eps_n \right) \leq D_2 \exp\left( - D_3 n h_n^d \eps_n^2  \right)
        \end{equation}
    \end{proposition}

    \underline{\textit{Completing the proof.}}
    Taking $ h_n = \Theta(n^{-\frac{1}{2\beta + d}})$ (minimizing the MISE $\Theta(h^{2\beta})+ \Theta(\frac{1}{nh^d})$), $D_1=1$ and $ n^{-\frac{\beta}{2\beta + d}} (\log n)^{1/2} \leq \eps_n/2 \leq \eps_0$ ensures that Proposition \ref{prop:rinaldo} applies for $\eps_n/2$. 
    Applying \eqref{eq:proof:thm:conc_sup3} to $h=h_n$ gives that almost surely $\|p-p_h\|_{\infty} = O(n^{-\frac{\beta}{2\beta + d}}) < \eps_n/2 $ for $n$ large enough.
    Now,  for $n$ large enough, 
    \begin{flalign*}\label{eq:proof:thm:conc_sup4}
        \dP(\|p-\what{p}_h\|_{\infty} > \eps_n) & \leq \dP(\|p-p_h\|_{\infty} > \eps_n/2) +  \dP(\|p_h-\what{p}_h\|_{\infty} > \eps_n/2) \nonumber \\
        & \leq 0 + D_2 \exp\left( - C_2 n^{-\frac{2\beta}{2\beta + d}} \eps_n^2  \right),
    \end{flalign*} which ends the proof of Theorem \ref{thm:conc_sup}.
\end{proof}

\subsection{Proof of Theorem \ref{th:conc_entropy}}
\begin{proof}[Proof of Theorem \ref{th:conc_entropy}]
    As stated in \eqref{eq:dvlp_von_mises}, the first step is to rigorously justify the Von Mises expansion. Note that since $(-y\log y)' = -\log y -1 $ and $(-y\log y)'' = - 1/y$, then for a given $\bx \in \cX$,
    \begin{equation}\label{eq:proof:th:conc_entropy1}
        |-p(\bx)\log p(\bx) + \what{p}_h(\bx)\log \what{p}_h(\bx) + (\log \what{p}_h(\bx)+1)(p(\bx)-\what{p}_h(\bx)) | \leq \left(\sup_{x \in \cX}\frac{1}{|\what{p}_h(\bx)|}\right) (p(\bx)-\what{p}_h(\bx))^2 ,
    \end{equation} and since $p$ is lower bounded by $p_{\min} >0$, then by Remark \ref{rem:positivity_ph}, for $n$ large enough, $\sup_{x \in \cX}\frac{1}{|\what{p}_h(\bx)|} \leq 2/p_{\min}$ and we can integrate of \eqref{eq:proof:th:conc_entropy1} over $\cX$ to indeed get
    \begin{flalign}\label{eq:proof:th:conc_entropy2}
        H(p) & = H(\what{p}_h) - \int_{\cX}(\log \what{p}_h(\bx) + 1)(p(\bx)-\what{p}_h(\bx))  \, \bdx + O\left( \int_{\cX} (p(\bx)-\what{p}_h(\bx))^2  \, \bdx \right) \nonumber \\
        & = - \int_{\cX} p(\bx) \log \what{p}_h(\bx) \, \bdx + O\left( \int_{\cX} (p(\bx)-\what{p}_h(\bx))^2 \, \bdx \right) \, .
    \end{flalign} This in turn implies that 
    \begin{flalign}\label{eq:proof:th:conc_entropy3}
        \what{H}_{\mathrm{vm}} - H(p) & = - \frac{2}{n} \sum_{i=n/2+1}^{n} \log \what{p}_h(\bx^{(i)}) + \int_{\cX} p(\bx) \log \what{p}_h(\bx) \, \bdx + O\left( \int_{\cX} (p(\bx)-\what{p}_h(\bx))^2 \, \bdx \right) \, .
    \end{flalign}
    The first two terms are the difference between an empirical mean and its expectation w.r.t $p$. 
    Recall that $n$ is large enough so that $\| p -\what{p}_h\|_{\infty} < p_{\min}/2$ (Remark \ref{rem:positivity_ph}). Hence, since $p$ is bounded on the compact set $\cX$, so is $\what{p}_h$. Every term in the sum $\sum_{i=n/2+1}^{n} \frac{2}{n} \log \what{p}_h(\bx^{(i)})$ is almost surely bounded by $c/n$ where $c>0$ is a constant. Azuma-Hoeffding inequality yields
    \begin{flalign}\label{eq:proof:th:conc_entropy4}
        \dP\left(\left|- \frac{2}{n} \sum_{i=n/2+1}^{n} \log \what{p}_h(\bx^{(i)}) + \int_{\cX} p(\bx) \log \what{p}_h(\bx) \, \bdx \right| > \eps_n/2 \right) & \leq 2\exp\left(- \frac{\eps_n^2}{8 \sum_{i=n/2+1}^{n} (c/n)^2}\right) \nonumber \\
        & =  2\exp\left(- \frac{\eps_n^2 n}{4c}\right) \nonumber \\
        & \leq C'_1 \exp\left(- C'_2 n^{1/2}\eps_n \right),
    \end{flalign} since $n^{1/2}\eps_n >1$ by assumption.
    The second part of the result comes from the inequality 
    \begin{flalign}\label{eq:proof:th:conc_entropy5}
        \dP\left( \int_{\cX} (p(\bx)-\what{p}_h(\bx))^2 \, \bdx > t \right) & \leq \dP\left( \sup_{\bx \in \cX} |p(\bx)-\what{p}_h(\bx)|  > \sqrt{t}/\mathrm{Vol(\cX)} \right), 
    \end{flalign} and appealing to Theorem \ref{thm:conc_sup} with a deviation $\sqrt{\eps_n}/C_4$ where $C_4>0$ is some constant (depending on $p_{\min}$ and $\mathrm{Vol(\cX)}$).
\end{proof}

\subsection{Proof of Corollary \ref{cor:conc_MI}}
\begin{proof}[Proof of Corollary \ref{cor:conc_MI}]
    The proof of Corollary \ref{cor:conc_MI} is a straightforward application of Theorem \ref{th:conc_entropy} to  $\what{H}_{\mathrm{vm}}(X,Y,\bZ)$, $\what{H}_{\mathrm{vm}}(X,\bZ)$, $\what{H}_{\mathrm{vm}}(Y,\bZ)$ and $\what{H}_{\mathrm{vm}}(\bZ)$.  The dimension-free rate comes from the assumption $\beta > 1 + d_{\bZ}/2$, which implies that $\frac{2 \beta}{2 \beta + (2+d_{\bZ})}$, $\frac{2 \beta}{2 \beta + (1+d_{\bZ})}$ and $\frac{2 \beta}{2 \beta + d_{\bZ}}$ are always larger than $1/2$.
\end{proof}

\subsection{Proof of Corollary \ref{cor:rate_CItest}}
\begin{proof}[Proof of Corollary \ref{cor:rate_CItest}]
    Let $I := I(X;Y \cond \bZ)$. The sum of type one and type two errors of $T$ is easily bounded for $n$ large enough by applying Corollary \ref{cor:conc_MI} as follows.  
    \begin{flalign*}
        \dP(\mbox{reject $H_0$} \cond H_0) + \dP(\mbox{accept $H_0$} \cond H_1) &  \leq \dP(|\what{I}_{\mathrm{vm}} - I| > I_{\min} /2) + \dP(|\what{I}_{\mathrm{vm}} - I| > I_{\min} /2) \\
        & \leq 2 C_1 \exp\left( - C_2 n^{1/2} I_{\min}/2 \right) \, .
    \end{flalign*} 
Finding $n$ such that the RHS of the above is less than $\alpha$ concludes the proof.
\end{proof}

\subsection{Proof of Theorem \ref{th:PC}}

\begin{proof}[Proof of Theorem \ref{th:PC}]

By definition, the number of CI tests required by this algorithm to recover the MEC is upper bounded by $2 \binom{m}{2} \sum_{i=0}^{\Delta-1} \binom{m-1}{i} = O(m^{\Delta+1})$.
Using Corollary \ref{cor:rate_CItest} and the union bound, the probability that at least one of the outputs of these CI tests is incorrect is less than:
$$
C_1 m^{\Delta+1} \exp\left( - C_2 n^{1/2} I_{\min}/2 \right) \,.
$$
Finding $n$ such that the RHS of the above is less than $\alpha$ gives 
$n \geq \Omega\left(\left( \frac{\Delta+1}{I_{min}} \log (m/\alpha) \right)^2\right)$ and concludes the proof.
\end{proof}

\subsection{Proof of Theorem \ref{th:GS}}
\begin{proof}[Proof of Theorem \ref{th:GS}]
Steps 1-2 conduct $O(m)$ CI tests in the worst case, hence $O(m^2)$ CI tests are needed to recover all Markov boundaries. Recall $\max_{X\in \bX} |\mathrm{MB}(X)| \leq \Gamma$. Then Step 3 needs $O(m \Gamma 2^\Gamma)$ CI tests. Finally, Step 4 performs $O(m \Gamma^2 2^\Gamma)$ tests at the worst case.
The rest of the steps of the algorithm do not require CI tests.
Therefore, GS requires $O(m^2+m \Gamma^2 2^\Gamma)$ number of CI tests.

Using Corollary \ref{cor:rate_CItest} and the union bound, the probability that at least one of the outputs of these CI tests is incorrect is less than:
$$
(m^2 + m \Gamma^2 2^\Gamma) m^{\Delta+1} \exp\left( - C_2 n^{1/2} I_{\min}/2 \right) \,.
$$
Finding $n$ such that the RHS of the above is less than $\alpha$ gives 
$n \geq \Omega\left( \frac{1}{I_{min}}\log \left(\frac{m^2+m\Gamma^2 2^\Gamma}{\alpha} \right)^2\right)$ and concludes the proof.
\end{proof}

\section{Further on numerical experiments}\label{app:experiments}

\subsection{Single conditional independence test}

\paragraph{Model} In our tests, $X$ and $Y$ are one dimensional and $\bZ=(Z_1, Z_2)$ is two dimensional. 
%
%
$X,Y,Z_1,Z_2$ are distributed on $[0,1]$ with same marginal distributions $p_{\beta}(x)=(\beta + 1.15) x^{\beta+0.15}\mathbf{1}_{[0,1]}(x)$ for some positive integer $\beta$. Note that this distribution -- often referred to as \emph{power law distribution} -- is $\beta-$Hölder smooth (see Definition \ref{def:holder}). Next, we denote by $\cU([0,1])$ the uniform law on $[0,1]$. We generate the data via inverse transform sampling as follows:
\begin{flalign*}
    U_{Z,1} & \sim \cU([0,1])\\
    U_{Z,2} & \sim \cU([0,1])\\
    U_{X} \cond ( U_{Z,1}, U_{Z,2})& \sim t_1 \delta_{U_{Z,1}} + t_2 \delta_{U_{Z,2}} + (1-t_1-t_2) \, \cU([0,1])\\
    U_{Y} \cond ( U_{Z,1}, U_{Z,2}, U_X) & \sim t_1 \delta_{U_{Z,1}} + t_2 \delta_{U_{Z,2}} + t_{xy} \delta_{U_{X}} + (1-t_1-t_2-t_{xy}) \, \cU([0,1]),
\end{flalign*} where $t_1, t_2, t_{xy}$ are non-negative real numbers such that $t_1 + t_2 + t_{xy} <1$. Then $X,Y,Z_1$ and $Z_2$ are obtained as follows: $X = (U_{X})^{\frac{1}{\beta + 1.15}}, Y = (U_{Y})^{\frac{1}{\beta + 1.15}}, Z_1 = (U_{Z,1})^{\frac{1}{\beta + 1.15}}$, and $Z_2 = (U_{Z,2})^{\frac{1}{\beta + 1.15}}$. Note that it suffices to take $t_{xy}=0$ to get conditional independence of $X$ and $Y$ given $\bZ$. In the case where $X \indep Y \cond \bZ$ we took $\beta = 3$ and $(t_1,t_2,t_{xy})=(0.2,0.2,0)$. For $X \not \indep Y \cond \bZ$ we took $\beta = 3$ and $(t_1,t_2,t_{xy})=(0.2,0.1,0.3)$ and $I_{\min} = 0.11$.

\paragraph{Parameters}
We present in Table \ref{tab:paras_CI} the parameters used for numerical experiments on CI tests.
\begin{table}[h]
    \centering
    \begin{tabular}{c|c|c|c}
       Method & Reference & Parameters & Values   \\
       \hline
       \multirow{3}{9em}{KDE + Von Mises} & \multirow{3}{5em}{This paper}  & $\beta$ & 3   \\
       & & $I_{min}$ & 0.11 \\
       & & $\gamma$ s.t. $h_n = \gamma n^{-\frac{1}{2\beta + 2+2}}$ & 0.35   \\
       \hline
       \multirow{2}{9em}{KNN + Von Mises} & \multirow{2}{9em}{\cite{poczos11a}} & $I_{min}$ &  $0.05$  \\
       & & number of bins $k$  & $\lfloor \sqrt{n}\rfloor$\\
       \hline
       \multirow{2}{9em}{MINE} & \multirow{2}{9em}{\cite{belghazi2018mutual}} & $I_{min}$& 0.11\\
       & & number of epochs & {\scriptsize 10 if $n \leq 100$, 50  if $n = 250$, 100 otherwise} \\
    \hline
    HSIC & \cite{zhang2012kernel} & statistical significance $\alpha$ & 0.001 \\
    \hline
    Gaussian & -- & statistical significance $\alpha$ & 0.05 \\
    \hline
    OT-based & \cite{akbari2023} & threshold $\delta$ & 1.7 \\
    \hline
    \end{tabular}
    \caption{Parameters for CI tests in numerical experiments}
    \label{tab:paras_CI}
\end{table}


\paragraph{Further comments on performance of MINE}
As shown in Figure \ref{fig:CI_test_power},  MINE (\cite{belghazi2018mutual}) performs very poorly in our experiments; the total error is close to 1. This is because the number of samples at which we work is way smaller than the number of samples required for the method to work, namely $\sim 2 \times 10^6$. 

\subsection{PC and GS algorithms}
\paragraph{The model}
For our experiments in Section \ref{sec: exp disc}, we used the following Structural Equation Model (SEM) to generate the data:
\begin{flalign*}
    X_1 &:=   U_1\\
    X_2 &:=  U_2\\
    X_3 &:=  X_1^2 + X_2 + U_3\\
    X_4 &:=  U_4\\
    X_5 &:=  0.5 \times X_1^2 - 0.5 \times X_4^2 + U_5\\
    X_6 &:= X_4^3 - X_5 + U_6 ,
\end{flalign*} where $U_i$ variables are i.i.d. power-law distributed with density $p_{\beta}(x)=(\beta + 1.15) x^{\beta+0.15}\mathbf{1}_{[0,1]}(x)$. It is clear from this SEM that the corresponding causal graph is the one displayed in Figure \ref{fig:true_causal_graph}.

\paragraph{Parameters}
Table \ref{tab:paras_algs} provides the parameters employed in numerical experiments for the PC and GS algorithms with various CI testers.

\begin{table}[h]
    \centering
    \begin{tabular}{c|c|c|c}
       Method & Reference & Parameters & Values   \\
       \hline
       \multirow{3}{9em}{KDE + Von Mises} & \multirow{3}{5em}{This paper}  & $\beta$ & 3   \\
       & & $I_{min}$ & 0.01 \\
       & & $\gamma$ s.t. $h_n = \gamma n^{-\frac{1}{2\beta + 2+2}}$ & 0.35   \\
      \hline
      \multirow{2}{9em}{KNN + Von Mises} & \multirow{2}{10em}{\cite{poczos11a}} & $I_{min}$ & $0.05$\\
      & & number of bins $k$ & $\lfloor \sqrt{n}\rfloor$ \\
      \hline
     \multirow{2}{9em}{MINE} & \multirow{2}{10em}{ \cite{belghazi2018mutual}} & $I_{min}$ & 0.01\\
      & & number of epochs &{\scriptsize 10 if $n \leq 100$, 50  if $n = 250$, 100 otherwise} \\
      \hline
    HSIC & \cite{zhang2012kernel} & statistical significance $\alpha$ & 0.001 \\
    \hline
    Gaussian & -- & statistical significance $\alpha$ & 0.05 \\
    \hline
    OT-based & \cite{akbari2023} & thresholds $\delta(d_{\bZ}=2, \ldots, 6)$ & [1.9, 1.8, 1.2, 0.4, 0.4] \\
    \hline
    \end{tabular}
    \caption{Parameters for PC and GS tests in numerical experiments}
    \label{tab:paras_algs}
\end{table}


\paragraph{Experiments for GS}

\begin{figure*}[h]
    \begin{subfigure}[t]{0.48\textwidth}
        \includegraphics[width=0.9\linewidth]{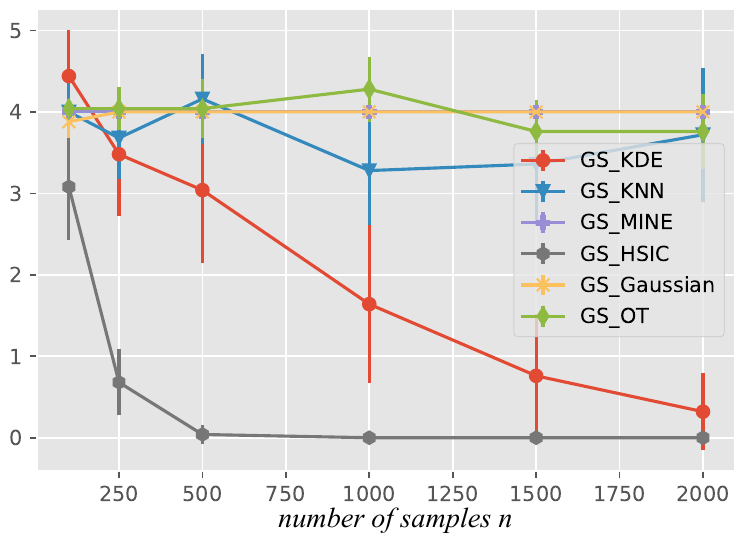}
        \caption{Overall loss of GS algorithm with various CI testers.}
        \label{fig:GS_loss}
    \end{subfigure}
    \hspace{0.01\textwidth}
    \hfill
    \begin{subfigure}[t]{0.48\textwidth}
        \includegraphics[width=0.9\linewidth]{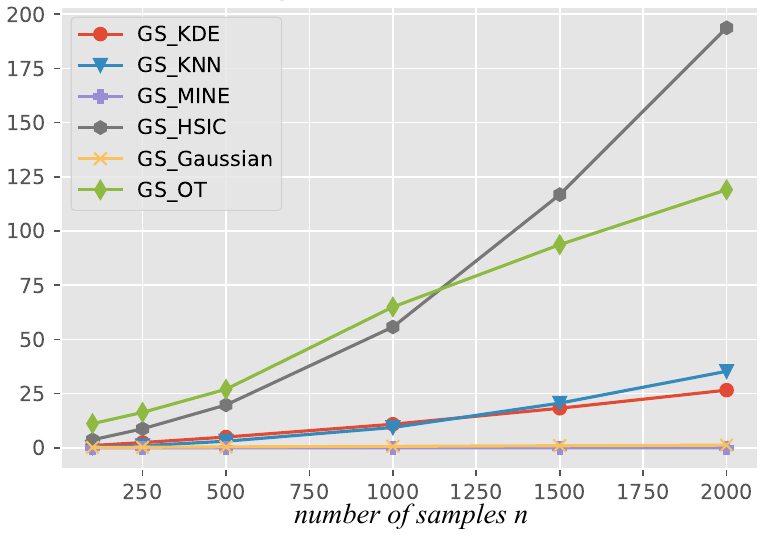}
        \caption{Runtime (s) of GS algorithm with various CI testers.}
        \label{fig:GS_time}
    \end{subfigure}
    \caption{Results of the numerical experiments for GS (on the x-axis: number of samples $n$)}
    \label{fig:results_GS}
\end{figure*}

Similar to Section \ref{sec: exp disc}, we conducted experiments to evaluate the performance of the GS algorithm when using VM-CI as a CI tester vs. other CI testers. 
To do so, we used the aforementioned SEM. Results for GS are shown in Figure \ref{fig:results_GS}.
Similar to the results observed for the PC algorithm, this figure illustrates that VM-CI surpasses the majority of approaches, with HSIC being the only exception. 
Nevertheless, analogous to PC, VM-CI competes with HSIC when the number of samples increases and offers significantly better computational efficiency than HSIC.

\end{document}